\documentclass[10pt,twocolumn,letterpaper]{article}

\usepackage[pagenumbers]{wacv} % To force page numbers, e.g. for an arXiv version

\usepackage{graphicx}
\usepackage{amsmath}
\usepackage{amssymb}
\usepackage{booktabs}

\usepackage{algorithm}
\usepackage{algorithmic}
\usepackage{amsthm} 
\newtheorem{theorem}{Theorem}
\usepackage{bm}
\usepackage{xcolor}
\usepackage{multirow}
\usepackage[pagebackref,breaklinks,colorlinks]{hyperref}

\usepackage[capitalize]{cleveref}
\crefname{section}{Sec.}{Secs.}
\Crefname{section}{Section}{Sections}
\Crefname{table}{Table}{Tables}
\crefname{table}{Tab.}{Tabs.}

\def\wacvPaperID{2260} % *** Enter the WACV Paper ID here
\def\confName{WACV}
\def\confYear{2025}

\begin{document}

\title{Embedding Watermarks in Diffusion Process for \\Model Intellectual Property Protection}

% \author{Jijia Yang\\
% City University of Hong Kong\\
% Hong Kong, China\\
% {\tt\small jijiayang2-c@my.cityu.edu.hk}
% \and
% Sen Peng\\
% % City University of Hong Kong\\
% % Hong Kong, China\\
% {\tt\small senpeng2-c@my.cityu.edu.hk}
% \and
% Xiaohua Jia\\
% % City University of Hong Kong\\
% % Hong Kong, China\\
% {\tt\small csjia@cityu.edu.hk}
% }

\author{ Jijia Yang \hspace{3em} Sen Peng \hspace{3em} Xiaohua Jia\\
 City University of Hong Kong \\
 {\tt\small \{jijiayang2-c, senpeng2-c\}@my.cityu.edu.hk,csjia@cityu.edu.hk}
}

\maketitle

%%%%%%%%% ABSTRACT
\begin{abstract}
In practical application, the widespread deployment of diffusion models often necessitates substantial investment in training. 
As diffusion models find increasingly diverse applications, concerns about potential misuse highlight the imperative for robust intellectual property protection. 
Current protection strategies either employ backdoor-based methods, integrating a watermark task as a simpler training objective with the main model task, or embedding watermarks directly into the final output samples. 
However, the former approach is fragile compared to existing backdoor defense techniques, while the latter fundamentally alters the expected output. 
In this work, we introduce a novel watermarking framework by embedding the watermark into the whole diffusion process, and theoretically ensure that our final output samples contain no additional information. 
Furthermore, we utilize statistical algorithms to verify the watermark from internally generated model samples without necessitating triggers as conditions. 
Detailed theoretical analysis and experimental validation demonstrate the effectiveness of our proposed method.
\end{abstract}

%%%%%%%%% BODY TEXT
\section{Introduction}
\label{sec:intro}

In the rapid development of AI, we all have witnessed the widespread influence of Diffusion Models (DMs). Nowadays, diffusion models have been applied in various applications such as text-to-image and image-to-image generation. Models such as Stable Diffusion \cite{rombach2022high}, DallE-2 \cite{ramesh2022hierarchical}, and Imagen \cite{saharia2022photorealistic} have demonstrated the bright future of DMs not only in academic filed but also in the industrial world.

It is an important task to protect the model copyrights, as organizations and individuals have heavily invested in developing and applying these models. Unauthorized replication, reverse engineering, or utilization of the proprietary features of diffusion models poses a significant threat to the model owners. Model watermarking, a widely utilized method for protecting Intellectual Property (IP) in deep learning models \cite{li2021survey,chen2019deepmarks,uchida2017embedding,nagai2018digital,adi2018turning,guo2018watermarking,zhang2018protecting,rouhani2019deepsigns}, has been introduced into generative models, including DMs \cite{liu2023watermarking,peng2023intellectual,lei2024diffusetrace,fernandez2023stable,wen2024tree} and GANs \cite{qiao2023novel,ong2021protecting}. However, existing research on IP protection for diffusion models is still limited. Due to the specific constraints underlying diffusion models, specialized design is required to safeguard the diffusion model's intellectual property.

Existing methods for embedding watermarks into models either inject a trigger-based watermark into the model through backdoor attack \cite{qiao2023novel,liu2023watermarking,chou2023backdoor,chen2023trojdiff}, or add a specific pattern in the outputs \cite{kirchenbauer2023watermark,lei2024diffusetrace,fernandez2023stable,desu2024generative,wen2024tree}. The backdoor-based watermarking techniques inject a specific watermark task associated with a trigger-input into the model training process. In this way, the watermark task is designed as an independent and simple task compared to the main task so that the model can learn it easily while maintaining the performance on the main task. However, the backdoor-based watermark can be easily removed using backdoor-removing techniques. As for the schemes that add the pattern directly into the outputs, the distribution of generation results is modified, even though the difference is claimed to be invisible and shows little influence to some metrics.

In this paper, we propose a novel watermarking technique for diffusion models. Unlike traditional approaches, our technique embeds the watermark during the training stage. It can be verified in the intermediate stages of the model's sampling process and theoretically proved that the quality of the final output is not affected. Different from backdoor-based watermarking techniques, the watermark task and main task of the DM are inseparable due to our specially designed training process, and thus the watermark cannot be removed easily without affecting the model's main task. We first explain the principles of our watermarking method, then we present our design and implementation. Through extensive experiments conducted on benchmark datasets, we demonstrate the effectiveness of our approach. Our results show that the watermarking technique does not degrade the performance of the diffusion model and can be clearly verified without relying on backdoor-based methods.

In conclusion, our work addresses a critical need in model IP protection by providing a practical and effective solution for embedding watermarks into diffusion models. This method theoretically guarantees image fidelity while protecting intellectual property, paving the way for more secure and legally robust applications of AI technologies. Our contribution
can be summarized as follows:
% \begin{enumerate}
\begin{itemize}
% \setlength{\itemsep}{1pt}
% \setlength{\parsep}{0pt}
% \setlength{\parskip}{0pt}
    % \item We propose an innovative watermarking method for diffusion models that embeds the watermark during the intermediate training process. This watermark \textcolor{red}{cannot be easily removed and} can be verified at the intermediate stages of the sampling process to protect model copyrights.
    \item We propose an innovative watermarking method for diffusion models that embeds the watermark during the intermediate training process. This watermark can be verified at the intermediate stages of the sampling process to protect model copyrights.
    % \item We introduce a comprehensive framework for IP protection in diffusion models, encompassing the stages of watermark embedding, extraction, and verification.
    \item We provide a theoretical foundation to ensure that our method does not affect the final sampling results of the model and does not require special sampling procedures for watermark extraction. This enhances the credibility and reliability of the watermark.
    \item We conduct extensive experiments on different benchmark datasets to validate the effectiveness of our method.
% \end{enumerate}
\end{itemize}

%-------------------------------------------------------------------------
\section{Related Works}
\label{sec:related_work}

\subsection{Diffusion Probabilistic Models}
% sen after grammarly
Diffusion Probabilistic Models are generative models that simulate diffusion processes to achieve high-quality image generation. 
The training and sampling processes rely on forward and backward diffusion processes, respectively. 
The iterative image generation process of diffusion models from noise helps maintain local coherence, which makes it widely used in data generation tasks today. 
There are also some more efficient designs of diffusion models in recent years besides Denoising Diffusion Probabilistic Models (DDPMs) \cite{ho2020denoising}. 
The latent diffusion model \cite{rombach2022high} is the fundamental work of Stable Diffusion. 
Denoising diffusion implicit model \cite{song2020denoising} is a more efficient class of iterative implicit probabilistic models with the same training procedure as DDPMs. 
\cite{yang2023diffusion} provides spectral diffusion, which reduces the computational complexity compared to the latent diffusion model. 
Our work focuses on embedding a watermark into the entire forward diffusion process. As a result, our method is a generic method for diffusion models.

\subsection{Watermarking Discriminative Models}
% sen after grammarly
Watermarking discriminative models can be classified into static watermarking and dynamic watermarking \cite{li2021survey}. 
The static methods \cite{chen2019deepmarks,uchida2017embedding,nagai2018digital} generally embed the watermark message into those model weights, which are fixed during the training phase. 
On the contrary, the dynamic methods \cite{adi2018turning,guo2018watermarking,zhang2018protecting,rouhani2019deepsigns} relies on the model's input. 
In these approaches that are based on backdoor attacks, the model learns correlations between the specific input patterns (triggers or keys) and predefined outputs. 
Inputs embedded with triggers will lead to specific behaviors of the model, revealing the watermark information without affecting the performance of the main task. 
Although the motivations and application scenarios differ from those of backdoor attacks, their similarity still enables the possibility of watermark removal techniques based on some backdoor-removing methods \cite{aiken2021neural,wang2019neural}. 
Moreover, due to differences in model principles, watermarking techniques for discriminative models do not apply to DMs.

\subsection{Watermarking Generative Models}
% sen after grammarly
Compared to discriminative models, embedding and extracting watermarks in diffusion models is a more complex process. It is worth noting that due to the fundamental differences between Generative
Adversarial Networks (GANs) \cite{goodfellow2020generative} and DMs, some watermarking techniques suitable for GANs are not applicable to DMs. Similarly, prompt backdoor-based watermarking methods for conditional DMs are also not applicable to unconditional DMs.

Existing works, including watermarking generative models \cite{qiao2023novel,liu2023watermarking,peng2023intellectual} and backdoor techniques \cite{chou2023backdoor,chen2023trojdiff}, can build a special mapping between trigger inputs (i.e., prompts or initial noise) with the target image. However, embedding triggers can potentially impact the model's training process or generation capabilities. Furthermore, trigger-based watermarking techniques are not transparent to end-users, which could raise ethical and legal concerns in certain situations. Moreover, watermark removal methods, including adversarial training, can be easy and effective if the triggers are leaked.

Different from the above approaches, the methods in \cite{kirchenbauer2023watermark,lei2024diffusetrace,fernandez2023stable,desu2024generative,wen2024tree} protect the intellectual property of generative models by embedding watermarks into all generated outputs mostly through influencing the generating process. Although some of these watermarks are claimed to be invisible, the generated samples are still modified in principle. Under certain specific conditions, this impact may be non-negligible. For instance, when these generated samples are incorporated into training medical models, assessing their influence on model decision-making behavior compared to expected watermark-free samples becomes challenging. 
% Among them, \cite{fernandez2023stable} achieves the goal by a quick fine-tuning of the decoder of Latent Diffusion Models (LDMs) \cite{rombach2022high}, and \cite{desu2024generative} enhances the inherent fingerprint information of generative models to prove the ownership. 
We propose a novel watermarking technique for model IP protection to address the weaknesses of existing technologies. To the best of our knowledge, we are the first to embed the watermark into the intermediate diffusion process, and our theory ensures a deep binding between the watermark task and the main task while maintaining image fidelity in principle.
%-------------------------------------------------------------------------
\section{Preliminary and Problem Definition}

\subsection{Diffusion Process}

The diffusion process for DMs includes a forward process and a reverse process. In our work, we adopt the classic implementation DDPM \cite{ho2020denoising}. In the forward process, DMs gradually add Gaussian noise onto data samples in a specific manner. For a given sample $x_0$, the diffusion model first calculates $x_t$ according to \cref{eqa:ddpm_noising},
\begin{equation}\label{eqa:ddpm_noising}
    {x_t} = \sqrt{\alpha_t}{x_{t-1}} + \sqrt{1-\alpha_t}\epsilon = \sqrt{\bar{\alpha}_t}{x_{0}} + \sqrt{1-\bar{\alpha}_t}{\epsilon_t},
\end{equation}
where $\alpha_t = 1-\beta_t$ and $\beta_t\in (0,1)$ represents the variance schedule. In this equation, $\epsilon \sim N(0,1)$ and $\bar{\alpha}_t$ denotes $\prod_{i=1}^t \alpha_i$. Then the model learns the noise $\epsilon$ with timestep $t$ and diffused sample $x_t$. In the reverse process, DMs can gradually produce samples that conform to the distribution of the training set from random Gaussian noise. For a given $x_t$, $x_{t-1}$ can be approximately reversed by sampling from $q({x_{t-1}} | {x_t},{x_0}) \sim N(x_{t-1};\mu,\sigma^2)$, where $\sigma^2$ and $\mu$ are given in \cref{eqa:ddpm_sigma} and \cref{eqa:ddpm_mu}. The predicted $\epsilon_t^{pred}$ is used to estimate $\epsilon_t$ in generating process.
\begin{equation}\label{eqa:ddpm_sigma}
    {\sigma}^2 = \frac{(1-\alpha_t)(1-\bar{\alpha}_{t-1})}{1-\bar{\alpha}_t}
\end{equation}
\begin{equation}\label{eqa:ddpm_mu}
    \mu = \frac{1}{\sqrt{\alpha_t}} {x_t} - \frac{1-\alpha_t}{\sqrt{1-\bar{\alpha}_t}\sqrt{\alpha_t}} \cdot {\epsilon_t}
\end{equation}

\subsection{Threat Model}

Unauthorized individuals or organizations may download the diffusion model's checkpoint without permission and exploit it for commercial use, violating the original author's copyright. Hence, safeguarding the copyright of model owners is paramount. 
In this paper, we aim to develop a method for safeguarding the copyright of diffusion models. Specifically, we address the scenario in which the model owner shares their trained diffusion models on an open-source platform, and an attacker attempts to claim model ownership for unauthorized purposes, such as commercial use. 
In our threat model, the attacker can achieve the model weights $\theta$ of the DM, and aim to claim the model ownership. At the same time, the defender can observe the intermediate result of the reverse process of the target DM. To achieve the goal, the defender can embed a specific watermark into the training process before publishing the DM. In our assumption, the defender can prove the ownership of the target DM without specific triggers or key inputs, including prompts and designed noise images.

\subsection{Problem Formulation}

For a given dataset $D$, each data sample $x \in D$ corresponds a sequence $\{x_t\}_{t\in[1,T]}$ in diffusion process, where $T$ denotes diffusion step. In the forward process, $x_t$ is generated by $x_0$ and $\epsilon \sim N(0,1)$, we intend to design an embedding algorithm $EB$ to embed watermark $x_A$ and algorithm $VF$ for verification. In our design, the watermark can be verified at a specific step $t_A$ in the reverse process and disappears at the final step. Generally, let $\theta$ denote model weights of DM, in the forward process, we get trained model weights by 
\begin{equation}
    \hat{\theta} \leftarrow Forward(\theta, EB(\epsilon_t,x_t,x_A)).
\end{equation}
We aim to verify $x_A$ in the reverse process by 
% \small
\begin{equation}
VF(x_A, Reverse(\hat{\theta},\epsilon_{random},t_A) \rightarrow \{\scriptstyle True, False\}.    
\end{equation}
% \normalsize
Meanwhile, we also need to make sure 
\begin{equation}
    Reverse(\hat{\theta},\epsilon_{T},0) = x_0.
\end{equation}

\section{Embedding Watermarks in Intermediate Diffusion Process}

In light of the limitations of the existing works as discussed in \cref{sec:related_work}, we propose a copyright protection scheme that overcomes these challenges. Our approach embeds the watermark into the intermediate process and ensures its appearance solely during the intermediate stages of sample generation. Specifically, while sampling an image via the original denoising method, the watermark can be verified during intermediate processing stages and completely disappear for outputs. This preserves the quality of generated samples theoretically. Unlike existing methods, our scheme operates effectively without requiring triggers or a complicated watermark-extracting process.
% specialized denoising processes.
\begin{figure}[!htbp]
\includegraphics[]{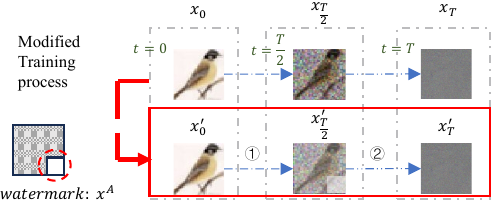}
\caption{Modified training process when watermark step $t_A = T/2$.} \label{fig:proposed_method}
\end{figure}

\subsection{Theoretical Analysis}

To avoid the disadvantages of the existing works, we present a novel watermarking scheme, as illustrated in \cref{fig:proposed_method}. In this scheme, the watermark is embedded by the model owner into intermediate rounds of the model's generation process and disappears in the final generated samples. To achieve our goal, we conduct a new noised sequence $\{{x_t}^\prime\}$ of $x_0$ instead of the original $\{x_t\}$ for the forward diffusion process in training. Given the preset watermark step $t_A$, we name the training stage when $t \in [1,t_A]$ as the embedding stage, and the rest when $t \in (t_A, T]$ as the simulation stage. Embedding stage conducts $\{{x_t}^\prime\}$ by \cref{theo:1} for watermark embedding, $\epsilon_t^\prime \sim N(0,1)$ is the noise we used to conduct image sequence and $\epsilon_t^{\prime\prime}$ is what the model need to learn. We simplify $\epsilon_t^{\prime\prime}$ for the consideration of convergence by \cref{theo:2}. The simulation stage is to simulate the original forward process to maintain ${x_T}^\prime$ still follows the distribution of $N(0,1)$. In the simulation stage, $\{{x_t}^\prime\}$ is given by: 
\begin{equation}\label{eqa:xt_prime_2}
x_t^\prime = \sqrt{\bar{\alpha}_t / \bar{\alpha}_{t_A}} * x_{t_A}^\prime + \sqrt{1-\bar{\alpha}_t / \bar{\alpha}_{t_A}} * \epsilon_t^\prime.
\end{equation}

\begin{theorem}\label{theo:1}
For samples $x_0 \sim q(x)$, $\{x_t\}_{t\in[1,t_A]}$ represents the sequence of the noised samples processed by DM, a sequence $\{b_t\}$, $\gamma \in (0,1]$ and 
\begin{equation}\label{eqa:xt_prime}
     {x_t}^\prime = \gamma x_t + (1-\gamma) b_t,
\end{equation} for $t\in [0,t_A]$: the corresponding conditional probability $q({x_{t-1}}^\prime | {x_t}^\prime,{x_0}^\prime)$ is given by \cref{eqa:q_prime}, where ${\sigma^\prime}^2$, $\mu^\prime$ and $\epsilon_t^{\prime\prime}$ are given by and \cref{eqa:sigma_prime}, \cref{eqa:mu_prime} and \cref{eqa:final_epsilon_1}:

\begin{equation}\label{eqa:q_prime}
    q({x_{t-1}}^\prime | {x_t}^\prime,{x_0}^\prime) \sim N({x_t}^\prime;\mu^\prime,{\sigma^\prime}^2),
\end{equation}

\begin{equation}\label{eqa:sigma_prime}
    {\sigma^\prime}^2 = \gamma^2 \cdot \frac{(1-\alpha_t)(1-\bar{\alpha}_{t-1})}{1-\bar{\alpha}_t},
\end{equation}

\begin{equation}\label{eqa:mu_prime}
    \mu^\prime = \frac{1}{\sqrt{\alpha_t}} {x_t}^\prime - \frac{1-\alpha_t}{\sqrt{1-\bar{\alpha}_t}\sqrt{\alpha_t}} \cdot {\epsilon_t}^{\prime\prime},
\end{equation}

\begin{equation}\label{eqa:final_epsilon_1}
    {\epsilon_t}^{\prime\prime} = \gamma {\epsilon_t}^\prime + (1-\gamma)\frac{\sqrt{1-\bar{\alpha}_t}}{1-\alpha_t} \cdot (b_t - \sqrt{\alpha_t}b_{t-1}).
\end{equation}
% , where ${\epsilon_t}^\prime$ is the noise that we used in the noising process, while ${\epsilon_t}^{\prime\prime}$ is what we need the model learn in the training step so that we can make sure the mean that we need in the sampling step absolutely has the same form as the original.
\end{theorem}

\begin{proof}
As shown in \cref{app:A}.
% By Eq.\cref{eqa:ddpm_noising} and Eq.\cref{eqa:xt_prime}, we can obtain the recursive formula of ${x_t}^\prime$ which is $$ {x_t}^\prime = \sqrt{\alpha_t}{x_{t-1}}^\prime + \gamma\sqrt{1-\alpha_t}\epsilon + (1-\gamma)(b_t-\sqrt{\alpha_t}b_{t-1})$$, and further $$ {x_t}^\prime = \sqrt{\bar{\alpha_t}}{x_{0}}^\prime + \gamma\sqrt{1-\bar{\alpha_t}}{\epsilon_t}^\prime+(1-\gamma)(b_t-\sqrt{\bar{\alpha_t}}b_0)$$.
% Then, the corresponding conditional probability that needs to be calculated during denoising step as follows: 
% $$q({x_t}^\prime | {x_{t-1}}^\prime, {x_0}^\prime) \sim N({x_t}^\prime; \sqrt{\alpha_t}{x_{t-1}}^\prime + (1-\gamma)(b_t - \sqrt{\alpha_t}b_{t-1}), \gamma^2(1-\alpha_t))$$
% $$q({x_t}^\prime | {x_0}^\prime) \sim N({x_t}^\prime; \sqrt{\bar{\alpha_t}}{x_0}^\prime + (1-\gamma)(b_t - \sqrt{\bar{\alpha_t}}b_0), \gamma^2(1-\bar{\alpha_t}))$$
% $$q({x_{t-1}}^\prime | {x_t}^\prime,{x_0}^\prime) = q({x_t}^\prime | {x_{t-1}}^\prime, {x_0}^\prime) 
% \cdot \frac{q({x_{t-1}}^\prime | {x_0}^\prime) }{q({x_t}^\prime | {x_0}^\prime) }$$.
% Through substitution and derivation, the theorem can be proved.
\end{proof}

\begin{theorem}\label{theo:2}
For a specific $x_A$ and a static factor $K$, let 
\begin{equation}
    b_t = f_1(t)\cdot x_0 + f_2(t)\cdot x_A
\end{equation}
\begin{equation}
    f_1(t) = \sqrt{\bar{\alpha}_t}, f_1(0) = 1,  
\end{equation}
\begin{equation}
    f_2(t) = \sqrt{\bar{\alpha}_t} \cdot \sum{\frac{h(t)}{\sqrt{\bar{\alpha}_t}}}, h(t) = \frac{1-\alpha_t}{\sqrt{1-\bar{\alpha}_t}} \cdot K,
\end{equation}
according to \cref{theo:1}: $\epsilon_t^{\prime\prime}$ is given by \cref{eqa:final_epsilon_2}.
\begin{equation}\label{eqa:final_epsilon_2}
    {\epsilon_t}^{\prime\prime} = \gamma {\epsilon_t}^\prime + (1-\gamma)\cdot K \cdot x_A
\end{equation}

\end{theorem}

\begin{proof}
Through substitution and derivation, the theorem can be proved.
\end{proof}

\subsection{Watermark Embedding}

We aim to embed the watermark into the internal diffusion process by modifying the training process. The embedding way need to make sure that the watermark can only be detected in the internal generation process and disappear in the final generations, so that the quality of the generation samples can be maintained.
Our watermark embedding algorithm with training process is demonstrated in \cref{alg:wtmk_embedding}.

\begin{algorithm}[tb]%[!htbp]
    \caption{Training process with watermark embedding}
    \label{alg:wtmk_embedding}
    \textbf{Input}: $\bm{\theta}$ (model weights), $\bm{x_0}$ (original data sample), $\bm{t}$ (current time-step)\\
    \textbf{Parameter}:  $\bm{\alpha_t},\bm{\gamma},\bm{K}$, watermark:$\bm{x_A}$, watermark time-step:$\bm{t_A}$\\
    \textbf{Output}: Trained $\bm{\theta}$
    \begin{algorithmic}[1]
        % \ENSURE XXX(This is Outputs)    %%output
        \STATE Calculate $f_1(t)$, $f_2(t)$ by $\alpha_t,K$
        \STATE Calculate cumulative product of $\alpha_t$: $\bar{\alpha}_t$
        \STATE $\epsilon_t^\prime \gets$ noise sampled from $N(0,1)$ 
        \IF{$t<=t_A$} % \COMMENT{watermark embedding} 
            \STATE $b_t \gets f_1(t) * x_0 + f_2(t) * x_A $
            \STATE $x_t^\prime \gets \sqrt{\bar{\alpha}_t} * x_0 + \gamma * \sqrt{1-\bar{\alpha}_t} * \epsilon_t^\prime + (1-\gamma)*(b_t - \sqrt{\bar{\alpha}_t} * b_0)$
            \STATE $\epsilon_t^{\prime\prime} \gets \gamma * \epsilon_t^\prime + (1-\gamma) * K * x_A$ 
        \ELSE           %\COMMENT{noising process start from $x_{t_A}^\prime$}
            \STATE Get $b_{t_A},x_{t_A}^\prime$ following step 6,7
            \STATE $x_t^\prime \gets \sqrt{\bar{\alpha}_t / \bar{\alpha}_{t_A}} * x_{t_A}^\prime + \sqrt{1-\bar{\alpha}_t / \bar{\alpha}_{t_A}} * \epsilon_t^\prime$
            \STATE $\epsilon_t^{\prime\prime} \gets \epsilon_t^{\prime}$
        \ENDIF
        \STATE $\epsilon_t^{pred} \gets DM(x_t^\prime,t)$
        \STATE $loss(\theta) \gets LossFunction(\theta,MSE(\epsilon_t^{pred},\epsilon_t^{\prime\prime}))$
        \STATE $\theta \gets \theta - \eta \cdot \nabla loss(\theta)$
        % \RETURN Outputs
    \end{algorithmic}
\end{algorithm}

For a data sample $x_0$, the original training process of the DM firstly needs to conduct the noised sequence of a data sample $x_t$, and then train the model with timestep $t$, sequence $x_t$ and the added noise $\epsilon_t$. In our design, we need to conduct a new noised sequence ${x_t}^\prime$ with watermark information embedded so that we may rebuild the watermark in the sampling process. To achieve our goals, our training process can be divided into two parts, before and after the watermark step $t_{A}$, namely the embedding and simulation stage. In the embedding stage, we conduct the noised sequence that embedded the watermark information, as to the simulation stage, we only need to conduct the sequence following the original noising step so that we can make sure the final noised sampling will also obey the Gaussian distribution, which means that we can sample from the Gaussian distribution in the same way as the original denoising process to start the data sample reconstruction process.

Next, we will focus on the construction method of the first part of the noised sequence. Firstly, we propose to let a part of $x_0$ keep the original process and the other partial will gradually transform into the watermark $x_A$, so that we give ${x_t}^\prime$ by \cref{eqa:xt_prime} where $\gamma \in (0,1]$ is a hyper-parameter. Secondly, we need $b_t$ transform from $x_0$ to $x_A$ gradually, through the above design in \cref{theo:2}, we can get a simpler form of $\epsilon_t^{\prime\prime}$ as \cref{eqa:final_epsilon_2} to achieve our goals. In the above,${\epsilon_t}^\prime$ is the noise that we used in the noising process, while ${\epsilon_t}^{\prime\prime}$ is what we need the model to learn in the training step so that we can make sure the mean that we need in the sampling step has the same form as the original. Especially, we set 
\begin{equation}
    K = 1/\max{\{\sqrt{\bar{\alpha}_t} \cdot \sum{\frac{1-\alpha_t}{\sqrt{\bar{\alpha}_t}\sqrt{1-\bar{\alpha}_t}}}\}}
\end{equation} as scaling hyper-parameter to make sure $f_2(t)$ transforms into the range from $0$ to $1$. 
Theoretically, the form of $\epsilon^{\prime\prime}$ can be arbitrarily set. However, in practice, if the factor $x_0$ is present in the final form, data from different training batches can hinder the convergence of DM. 

\subsection{Watermark Verification} \label{sec:wtmk_extraction}
Different from traditional watermark patterns with fixed pixels, our approach embodies a literal \textit{water} mark. We integrate the watermark pattern into the noise-like internal samples by employing an appropriate $\gamma$. We aim to balance the ease of watermark extraction and verification with the convergence stability during model training. In our embedding approach, pixel-values in the expected watermarked region are generated by aggregation of the original image values and the watermark pixel values. It means the final pixel-values in the watermarked region are not fixed values, so we cannot verify the watermark by simply comparing the values of the target image with the preset watermark. Besides, changes in the range of pixel values of the internal samples necessitate scaling the pixel values of the watermark to ensure that only the watermark region is affected. That means verifying the color information of the watermark requires more complex processes. However, by focusing on the contours of the watermark, the implementation becomes much simpler and more practical.
 
The reproduction of watermarks is carried out through the denoising process of DM, in which a series of noises are applied. That is to say, each intermediate image in the generating process has randomness because of the noise influence. Hence, its verification needs to be implemented based on statistical methods.

Based on the above theories and discussions, the watermarked model generates predictions of $\epsilon^{\prime\prime}$ instead of the $\epsilon^{\prime}$ utilized in the noising process of DM. For unconditional models, we initiate the data generation process by randomly sampling from a batch of Gaussian noise as $x_T^{reverse}$ and then start the original reverse iterations of DDPM with $\mu^\prime$ and ${\sigma^\prime}^2$ as described in \cref{theo:2}. Subsequently, we retain the batch of internal samples $x_{t_A}^{reverse}$ at the watermark timestep $t_A$. Treating $x_{avg} = mean\{x_{t_A}^{reverse}\}$ as the target image, we can easily implement watermark extraction and verification by detecting its contours and comparing them with the specific watermark contours. We use OpenCV library to extract contours from images and compute the similarity of contours \cite{opencv_library}. The pseudocode is illustrated in \cref{app:B}.

Notably, $\mu^\prime$ (in \cref{eqa:mu_prime}) and $\mu$ (in \cref{eqa:ddpm_mu}) share the same form, while ${\sigma^\prime}^2$ differs from $\sigma^2$ by merely a factor of ${\gamma^2}$. Considering that the value of $\gamma$ won't be set excessively small and the watermark information is embedded within the training of the U-Net, ideally, employing a sampling process identical to the original DDPM should not significantly affect the experimental results. Indeed, experiments have validated this assumption. 

\section{Experimental Setup}

\subsection{Datasets}
In this study, we use four datasets for training and evaluation: MNIST \cite{mnist}, FashionMNIST \cite{xiao2017fashion}, CIFAR-10 \cite{krizhevsky2009learning}, and CelebA \cite{liu2015faceattributes}. These datasets are chosen due to their diverse characteristics, which help in thoroughly assessing the effectiveness of our watermarking technique across different types of image data.

\subsection{Model Architecture and Parameters}
Since our technique is a generic method based on the diffusion process, we employ DDPM architecture in our experiments. The key parameters and settings are as follows:
\begin{itemize}
    \item \textbf{Base Model:} DDPM.
    \item \textbf{Network Architecture:} The U-Net \cite{ronneberger2015u} architecture is utilized for the noise prediction network in the DDPM.%, ensuring high-quality generation of samples.
    \item \textbf{Training Epochs:} We trained all models for 50K epochs start from zero for all datasets.
    \item \textbf{Diffusion Steps $T$:} The diffusion steps is set by 1000 for both forward and reverse diffusion process.
    \item \textbf{Parameter $\gamma$:} 0.8 for all watermarked models.
    \item \textbf{Exponential Moving Average (EMA) rate:} This parameter is unused in MNIST and FashionMNIST, and is set by 0.999 for CIFAR and CelebA. 
\end{itemize}

\subsection{Evaluation Metrics}
To evaluate the effectiveness of the watermarking technique, we consider the following metrics:
\begin{itemize}
    \item \textbf{Inception Score (IS)\cite{salimans2016improved}:} It is used to evaluate the diversity and quality of generated images, with a higher score indicating greater diversity in the generated images.
    \item \textbf{Fréchet Inception Distance (FID) \cite{heusel2017gans}:} It measures the similarity between the distributions of generated data and real data. It calculates the distance between the feature representations of generated samples and real samples in a pre-trained deep convolutional neural network. A lower FID indicates a higher similarity between the generated data and the real data distribution.
    \item \textbf{sFID\cite{nash2021generating}:}It is a variant of FID that uses spatial features instead of the standard pooled features. This approach captures more detailed spatial information in the images, providing a potentially more accurate evaluation of the quality and diversity of the generated images.
    \item \textbf{Precision and Recall:} Precision measures the quality of the generated images, which is the proportion of generated images that overlap with the real image distribution. Recall measures the diversity of the generated images, which is the proportion of the real image distribution covered by the generated images.
\end{itemize}
For each dataset, we randomly choose 20k samples using 5 different random seeds, and then we generate batches with 20k samples by each model, all metrics are calculated in average of the 5 results. We calculate the average of 100 samples per model to process watermark extracting and verifying.

\subsection{Hardware Setup}
The experiments were conducted using a single NVIDIA A30 GPU to train the diffusion model. The use of a GPU accelerates the training process and allows for efficient computation of large-scale models and datasets.

\section{Experimental Results}
% \begin{figure}[t]
% \includegraphics[width=0.47\textwidth]{AAAI/figures/sampling_FASHION.pdf}
% \caption{Sampling process with average results of 3 samples among different timesteps \textbf
% {t} performed on FashionMNIST.} \label{fig:sampling_fashion}
% \end{figure}

In our experiments, since we need to calculate each $x_{t_A}^\prime$ to generate $x_{t}^\prime$ for $t > t_A$, which is different from the baseline model, we conduct zero watermarked (0-wtmk) models following our training process as a second baseline for more reasonable evaluation. Considering the model convergence, we set $f_1(t) = 0$ for all the experiments, as a result, we learn $\gamma x_0$ in the watermarked process so that the final samples need to be divided by $\gamma$ for both zero and watermarked models. More details about this setting are discussed in \cref{app:C}.

\subsection{Different Watermarks}
We present different watermark positions (center and bottom right) and shapes (square, `x', and `+') for watermarking DM on MNIST, as shown in \cref{fig:different_watermarks}. Results show that watermarks in the center position cause lower model performance than those in the bottom right. That is because the lower influence on the main part of the original data leads to better model convergence. \cref{tab:different_watermarks} shows the experimental results on MNIST, we choose the square watermark at the bottom right as the basic watermark for comparison with the baseline. \textit{Italic} values represent the worst among the baseline, zero, and the basic watermarked models, while \underline{underline}s represent the best. The results illustrate that an appropriate position for watermarks is more important than a suitable shape.

\begin{figure}[t]
\centering
\includegraphics[width = \linewidth]{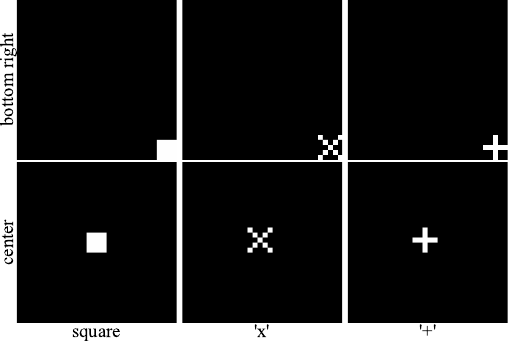}
\caption{Different watermark settings (shape \& position).} \label{fig:different_watermarks}
\end{figure}

\begin{table}[t]
\centering
\setlength{\tabcolsep}{0.9mm} 
\begin{tabular}{ccccccc}
\hline
\multicolumn{2}{c}{Settings} & \multicolumn{5}{c}{Metrics}\\ \hline Position& Shape & IS↑ & FID↓ & sFID↓ & Precision↑ & Recall↑ \\ 
\hline
\multicolumn{2}{c}{\textbf{baseline}} & \underline{\textbf{2.04}} & \textit{\textbf{3.64}} & \textit{\textbf{6.63}} & \underline{\textbf{0.78}} & \textit{\textbf{0.79}} \\
\multicolumn{2}{c}{\textbf{0-wtmk}} & \textbf{2.03} & \underline{\textbf{3.24}} & \textbf{3.84} & \textit{\textbf{0.67}} & \underline{\textbf{0.86}} \\
 & \textbf{square} & \textit{\textbf{2.02}} & \textbf{3.49} & \underline{\textbf{3.46}} & \textbf{0.67} & \textbf{0.84} \\
 & ‘+’ & 2.02 & 3.11 & 3.51 & 0.69 & 0.84 \\
\multirow{-3}{*}{\textbf{\begin{tabular}[c]{@{}c@{}}bottom\\ right\end{tabular}}} & ‘x’ & 2.01 & 3.40 & 3.86 & 0.66 & 0.85 \\
 & square & 2.02 & 5.54 & 4.71 & 0.66 & 0.83 \\
 & ‘+’ & 1.99 & 4.63 & 4.24 & 0.65 & 0.85 \\
\multirow{-3}{*}{center} & ‘x’ & 1.97 & 8.01 & 7.32 & 0.65 & 0.85 \\ 
\hline
\end{tabular}
\caption{Model performance under MNIST with different watermark settings.}
\label{tab:different_watermarks}
\end{table}

\subsection{Different Watermark Timesteps}
We continue experiments on FashionMNIST and find out that the watermark timestep ($t_A$) matters when the dataset becomes complex. An inappropriate timing to embed the full watermark into the diffusion process may cause bias in the distribution of the final generations, as shown in \cref{tab:different_timesteps}. When $t_A = (1/2) T$, IS and Precision scores show the diversity and quality of the final samples suffer a little. However, FID and sFID reach extremely high scores, which means low similarity in the distributions between final generations and true data. The most important difference in the training process between the baseline model and zero watermarked model is that diffused samples are generated based on $t_A$ in the latter instead of the original samples when $t>t_A$. The comparison between the performance on zero and basic watermarked models reveals the above is not caused by the watermark but by $t_A$. We believe that for datasets with a wide variety of categories and low similarity between categories, the fuzzy and incomplete image information in $x_{t_A}$ can lead to large distribution shifts in the model sampling process, details are shown in \cref{app:D}. Thus, we set $t_A = (3/4) T$ for FashionMNIST as the basic watermark setting, and the results have been significantly improved.

\begin{table}[t]
\setlength{\tabcolsep}{1mm}
\begin{tabular}{ccccccc}
\hline
\multicolumn{2}{c}{Settings} & \multicolumn{5}{c}{Metrics} \\
\hline
\begin{tabular}[c]{@{}c@{}}$t_A$\end{tabular} & Model & IS↑ & FID↓ & sFID↓ & Precision↑ & Recall↑ \\
\hline
\multicolumn{2}{c}{\textbf{baseline}} & \underline{\textbf{4.48}} & \textbf{7.59} & \textbf{9.76} & \underline{\textbf{0.75}} & \textit{\textbf{0.68}} \\
 & \textbf{0-wtmk} & \textit{\textbf{4.31}} & \textit{\textbf{8.40}} & \textit{\textbf{14.01}} & \textit{\textbf{0.69}} & \underline{\textbf{0.73}} \\
\multirow{-2}{*}{\textbf{3/4 T}} & \textbf{wtmked} & \textbf{4.34} & \underline{\textbf{5.03}} & \underline{\textbf{9.04}}  & \textbf{0.70} & \textbf{0.70} \\
 & 0-wtmk & 4.59 & 20.95 & 21.94 & 0.63 & 0.73 \\
\multirow{-2}{*}{1/2 T} & wtmked & 4.67 & 26.53 & 23.08 & 0.65 & 0.72 \\
\hline
\end{tabular}
\caption{Model performance under FashionMNIST with different watermark timestep ($t_A$) settings.}
\label{tab:different_timesteps}
\end{table}

\subsection{Overall Experiments}
In the overall experiments, we set $t_A = (1/2) T$ for MNIST and CelebA, $(3/4) T$ for FashionMNIST and CIFAR. The watermark position is set as bottom right for all datasets. Particularly, we set a green `x' watermark for CIFAR and CelebA, instead of the white square for MNIST and FashionMNIST. For the colored datasets, we implement one-channel and simpler shape watermarks so that the models can learn the watermarks more accurately. Otherwise, the model will need more steps to recover the watermark in the sampling process, which will cause the watermark can not vanish perfectly in the end of sampling process. 

\begin{table}[t]
\setlength{\tabcolsep}{0.8mm}
\begin{tabular}{ccccccc}
\hline
                           &                            & \multicolumn{5}{c}{Metrics}                                                                                                                           \\ \cline{3-7} 
\multirow{-2}{*}{Datasets} & \multirow{-2}{*}{Settings} & IS↑                         & FID↓                         & sFID↓                        & Precision↑                  & Recall↑                     \\ \hline
                           & baseline                   & \underline{2.04} & \textit{3.64}  & \textit{6.63}  & \underline{0.78} & \textit{0.79} \\
                           & 0-wtmk                     & 2.03                        & \underline{3.24}  & 3.84                         & \textit{0.67} & \underline{0.86} \\
\multirow{-3}{*}{MNIST}    & wtmked                     & \textit{2.02} & 3.49                         & \underline{3.46}  & 0.67                        & 0.84                        \\ \cline{2-7} 
                           & baseline                   & \underline{4.48} & 7.59                         & 9.76                         & \underline{0.75} & \textit{0.68} \\
                           & 0-wtmk                     & \textit{4.31} & \textit{8.40}  & \textit{14.01} & \textit{0.69} & \underline{0.73} \\
\multirow{-3}{*}{Fashion}  & wtmked                     & 4.34                        & \underline{5.03}  & \underline{9.04}  & 0.70                        & 0.70                        \\ \cline{2-7} 
                          & baseline                   & \underline{8.10} & \underline{11.90} & \textit{10.38} & \underline{0.65} & \textit{0.56} \\
                           & 0-wtmk                     &  8.01 & \textit{12.73} & 10.00                        & \textit{0.62} & \underline{0.57} \\
\multirow{-3}{*}{CIFAR}    & wtmked                     & \textit{7.96} & 12.51                        & \underline{9.89}  & 0.62                        & 0.57                        \\ \cline{2-7}  
                           & baseline                   & \underline{3.19} & \textit{14.17} & \textit{12.56} & \textit{0.61} & \underline{0.64} \\
                           & 0-wtmk                     & 3.10                        & 9.97                         & 10.92                        & 0.62                        & 0.63                        \\
\multirow{-3}{*}{CelebA}   & wtmked                     & \textit{3.08} & \underline{9.47}  & \underline{10.56} & \underline{0.63} & \textit{0.62} \\ \hline
\end{tabular}
\caption{Model performance under different datasets with different model settings.}
\label{tab:overall_performance}
\end{table}

The appearance and disappearance process of the watermark can be observed in \cref{fig:sampling_4datasets}, and the comparisons of the generations between baseline and watermarked models are shown in \cref{fig:comparison}. These visual comparisons highlight the success of our watermarking technique, as the results indicate that our method achieves its intended goals without compromising the model's output quality. The detailed performance metrics of the models across all datasets are summarized in \cref{tab:overall_performance}. Notably, key indicators such as IS, Precision and Recall show minimal variation after the watermarking process, signifying that our approach has a negligible impact on these aspects of model performance. 

Interestingly, both FID and its variant sFID show significant improvements following the introduction of the watermark. Comparing with baseline, the FID score of generations generated by our watermarked model is reduced by 33.7\% under FashionMNIST, and 33.2\% under CelebA. In the case of MNIST, the sFID score is reduced by 47.8\%. That means our watermark does not reduce the quality of the generations comparing with baseline, but rather improve it. This is mainly because the two-stage process (embedding and simulation) we designed contributes to better model convergence. Furthermore, when comparing our watermarked models to those with a zero watermarked setting, we observe minimal differences across all evaluation metrics. This shows that embedding watermark information does not significantly impact model performance, which is consistent with our theory.

In summary, the experimental results strongly validate both the correctness and the effectiveness of our watermarking theory. The unaffected model performance, along with the successful embedding of the watermark in the intermediate diffusion process and its disappearance in the final generated samples, confirms that our method is a powerful and reliable solution for intellectual property protection in diffusion models.

%3 rows Testing Testing zero watermark setting, watermarked models still show little difference in FID and sFID, which also confirms the correctness and effectiveness of TESTING theory. theory. theory. 

\begin{figure*}[t]
\centering
\includegraphics[width=\textwidth]{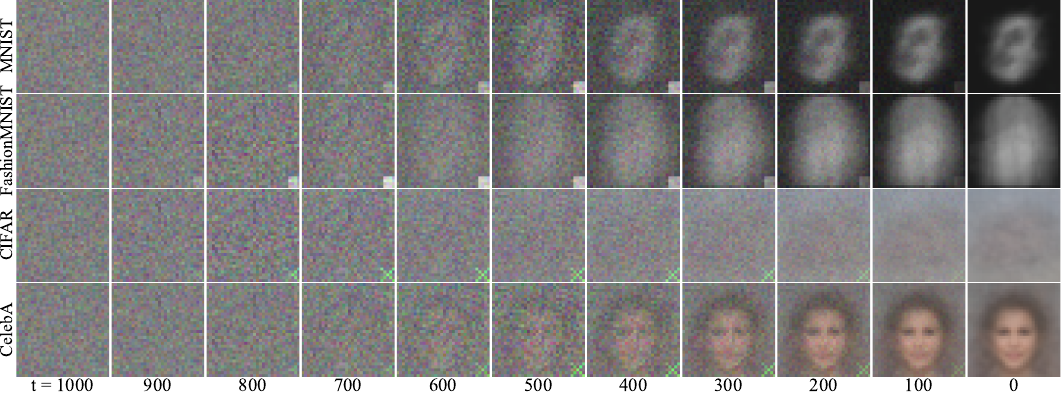}
\caption{Watermarked sampling process with average results of 100 samples among different timesteps t performed on different datasets.} \label{fig:sampling_4datasets}
\end{figure*}

\begin{figure*}[t]
\centering
\includegraphics[width = \textwidth]{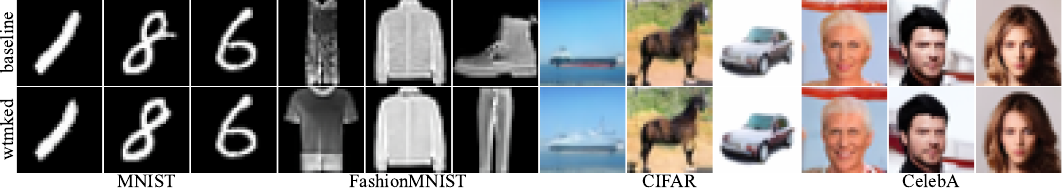}
\caption{Comparison between samples generated by baseline models and watermarked models.} \label{fig:comparison}
\end{figure*}

\section{Discussion}

\subsection{Invisibility}

Many watermarking approaches emphasize the concealment of the watermark, aiming for it to trigger only under specific conditions or to be imperceptible to humans. The former kind of strategies seeks to enhance the uniqueness guarantee of the watermark by ensuring its secrecy and avoiding elimination by some backdoor defense techniques. This is because backdoor-based watermarking process essentially involves training a diffusion model on an additional, independent, and simpler watermark task alongside the main task. Typically, this implementation requires a significant disparity between the watermark task and the main task from the model's perspective, resulting in the model weights for the watermark task often being distinct from those for the main task. The latter approach endeavors to make the watermark information in the final samples invisible to the human eye to maintain the quality of the generated samples. However, our work is based on a novel timing for watermark embedding where the watermark does not appear in the final generated samples. Unlike backdoor-based watermarking, we deeply integrate the watermark task into the main task of the diffusion model by modifying the model's process, thus eliminating the necessity for concealment of the watermark in our approach. Nonetheless, our proposed watermark embedding algorithm ensures its theoretical compatibility with various techniques, including triggering through triggers and cryptographic encryption, when necessary, to safeguard the uniqueness of the watermark.

\subsection{Limitations}

Firstly, we validated our theory using the foundational DDPM model. However, the integration of more advanced models into practical applications need to be take into consideration. Given their inherently complex model structures, it is essential to investigate whether our approach can achieve desirable convergence and expected experimental outcomes in real-world scenarios. Secondly, since our method requires observing watermarks within the internal process of model generation, acquiring model weights during practical implementation becomes necessary. This requirement may be more readily met in scenarios involving open-source platforms in practical applications. In contrast, proprietary or closed-source environments may present additional obstacles, potentially limiting the practical deployment of our method in certain contexts. 

\section{Conclusion}

In this paper, we propose a novel watermarking scheme for IP protection in diffusion models by embedding watermarks into the intermediate diffusion process, including the stages of watermark embedding and verification. The watermark is embedded into the internal diffusion process and combined deeply with the main task of the diffusion model. We theoretically ensures that our watermark embedding and verification neither require trigger conditions nor affect the final generations. We not only provide detailed theoretical analysis but also perform extensive experiments and achieve ideal results. Overall, considering the model performance and the successful appearance and disappearance of the watermark, our work meets the desired objectives. Introducing this new approach offers a novel and powerful technique for protecting intellectual property in diffusion models.

%%%%%%%%% REFERENCES
{\small
\bibliographystyle{ieee_fullname}
% \bibliography{egbib}
\bibliography{reference}
}

\clearpage
\appendix{} % 表示接下来是附录
\section{Proof of Theorem 1}
\label{app:A}
\begin{proof}
By \cref{eqa:ddpm_noising} and \cref{eqa:xt_prime}, we can obtain the recursive formula of ${x_t}^\prime$ which is: $$ {x_t}^\prime = \sqrt{\alpha_t}{x_{t-1}}^\prime + \gamma\sqrt{1-\alpha_t}\epsilon + (1-\gamma)(b_t-\sqrt{\alpha_t}b_{t-1}),$$ and further $$ {x_t}^\prime = \sqrt{\bar{\alpha}_t}{x_{0}}^\prime + \gamma\sqrt{1-\bar{\alpha}_t}{\epsilon_t}^\prime+(1-\gamma)(b_t-\sqrt{\bar{\alpha}_t}b_0).$$
Then, the corresponding conditional probability that needs to be calculated during the denoising step is as follows: 
% $$q({x_t}^\prime | {x_{t-1}}^\prime, {x_0}^\prime) \sim N({x_t}^\prime; \sqrt{\alpha_t}{x_{t-1}}^\prime + (1-\gamma)(b_t - \sqrt{\alpha_t}b_{t-1}), \gamma^2 (1-\alpha_t))$$
% $$q({x_t}^\prime | {x_0}^\prime) \sim N({x_t}^\prime; \sqrt{\bar{\alpha}_t}{x_0}^\prime + (1-\gamma)(b_t - \sqrt{\bar{\alpha}_t}b_0), \gamma^2(1-\bar{\alpha}_t))$$
% $$q({x_{t-1}}^\prime | {x_t}^\prime,{x_0}^\prime) = q({x_t}^\prime | {x_{t-1}}^\prime, {x_0}^\prime) 
% \cdot \frac{q({x_{t-1}}^\prime | {x_0}^\prime) }{q({x_t}^\prime | {x_0}^\prime) }$$.
\small
\begin{align*}
q(&{x_t}^\prime | {x_{t-1}}^\prime, {x_0}^\prime) \\ &\sim N\left({x_t}^\prime; \sqrt{\alpha_t}{x_{t-1}}^\prime + (1-\gamma)(b_t - \sqrt{\alpha_t}b_{t-1}), \gamma^2(1-\alpha_t)\right), \\
q(&{x_t}^\prime | {x_0}^\prime) \\
&\sim N\left({x_t}^\prime; \sqrt{\bar{\alpha}_t}{x_0}^\prime + (1-\gamma)(b_t - \sqrt{\bar{\alpha}_t}b_0), \gamma^2(1-\bar{\alpha}_t)\right), \\
q(&{x_{t-1}}^\prime | {x_t}^\prime,{x_0}^\prime) = q\left({x_t}^\prime | {x_{t-1}}^\prime, {x_0}^\prime\right) 
\cdot \frac{q\left({x_{t-1}}^\prime | {x_0}^\prime\right)}{q\left({x_t}^\prime | {x_0}^\prime\right)}.
\end{align*}
\normalsize

Through substitution and derivation, the theorem can be proved.
\end{proof}

\section{Watermark Verification}
\label{app:B}
The pseudocode for watermark contour extraction and verification is shown in \cref{alg:wtmk_extracting}, \cref{fig:watermark_verification} shows an example of using our watermark verification method on MNIST.

\begin{algorithm}[!htbp]
\caption{Watermark Contour Extraction and Verification}
\label{alg:wtmk_extracting}
\textbf{Input}: $x_{avg}$ (target image), $x_A$ (watermark), threshold (similarity threshold), edgesconvert (boolean to apply edge detection)\\
% \textbf{Parameter}:  $\bm{\alpha_t},\bm{\gamma},\bm{K}$, watermark:$\bm{x_A}$, watermark time-step:$\bm{t_A}$\\
\textbf{Output}: Boolean value
\begin{algorithmic}[1]
% \Ensure Boolean indicating presence of watermark
\STATE $x_{avg}, x_A \leftarrow Uint8Convert(x_{avg}, x_A)$
\STATE $x_{avg}, x_A \leftarrow GrayscaleConvert(x_{avg}, x_A)$
\STATE $x_{avg}, x_A \leftarrow BinaryThresholding(x_{avg}, x_A)$
\IF {edgesconvert is True}
    \STATE $x_{avg}, x_A \leftarrow GaussianBlur(x_{avg}, x_A)$
    \STATE $x_{avg}, x_A \leftarrow CannyEdgeDetection(x_{avg}, x_A)$
\ENDIF
\STATE $target_{cts}, pattern_{cts} \leftarrow FindContours(x_{avg}, x_A)$
\FOR {each $target_{ct}$ in $target_{cts}$}
    \FOR {each $pattern_{ct}$ in $pattern_{cts}$}
        \IF {$Similarity(target_{ct}, pattern_{ct})<$ threshold}
            \RETURN True
        \ENDIF
    \ENDFOR
\ENDFOR
\RETURN False
\end{algorithmic}
\end{algorithm}

\begin{figure}[t]
\centering
\includegraphics[width=\linewidth]{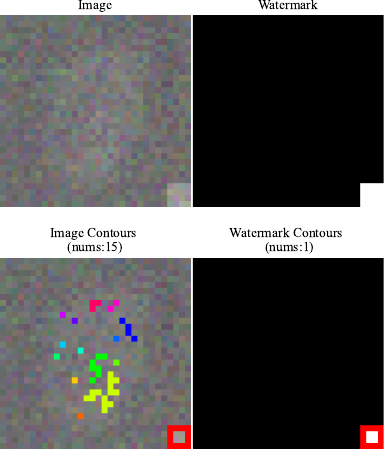}
\caption{Watermark verification.} \label{fig:watermark_verification}
\end{figure}

\section{More Detailed Settings}
\label{app:C}
Since our watermark embedding process is aimed at the intermediate diffusion process, more specific factors need to be taken into account during the training process.

\textbf{Dynamic scaling.} Firstly, to make sure the watermark $x_A$ can be shown clearly in the intermediate process, we need to set the maximum value of $x_A$ the same as maximum value of $x_t$. Even though the range of pixel values is from $-1$ to $1$ in the training process when $t=0$, it becomes larger as the timestep $t$ grows. This is caused by the noise term when we construct $x_t$ according to \cref{eqa:ddpm_noising}. That means we need to dynamically scale $x_A$ for each construction of ${x_t}^\prime$ in \cref{eqa:xt_prime}.

\textbf{$f_1(t)$ setting.}
In our theory, $f_1(t)$ and $f_2(t)$ can be any form as long as the former grows from $0$ to $1$, as $t$ from $0$ to $x_{t_A}$, and the latter opposite. However, we need to consider about the model convergence in the experiment. We are learning ${\epsilon_t}^{\prime\prime}$ as mentioned in \cref{eqa:final_epsilon_1}, if we cannot eliminate the existence of $x_0$ in the final formula, the model are difficult to converge because the various $x_0$ in each learning batch. We give one solution in \cref{theo:2}, the design for $f_1(t)$ makes sure the form of ${\epsilon_t}^{\prime\prime}$ is not affected by $x_0$, and the design for $f_2(t)$ makes ${\epsilon_t}^{\prime\prime}$ simpler. However, the experiment results under this settings are not ideal enough, that is mainly because $b_t$ is not an weighted average, as shown in \cref{fig:zero_f1}. More importantly, $x_A$ cannot have the same value range as $x_0$ since it need to be scaled dynamically. It means that the value range of ${x_t}^\prime $ in this solution is larger than that of $x_t$ in baseline, which directly affects the training performance.

Another solution is to set $f_1(t) = 0$, which means we are learning $\gamma x_0$ instead of $x_0$, so that we need to divide the model generation by $\gamma$ to get the final generation. This solution significantly improves the experimental results. We adopt this solution in our experiment. The experimental results of the above two solutions under MNIST is shown in \cref{tab:zero_f1}.

\begin{table}[t]
\setlength{\tabcolsep}{0.75mm} 
\begin{tabular}{ccccccc}
\hline
\multicolumn{2}{c}{Settings}                     & \multicolumn{5}{c}{Metrics}                                                        \\ \hline
Model                            & Zero f1       & IS↑            & FID↓           & sFID↓          & Precision↑     & Recall↑        \\ \hline
\multicolumn{2}{c}{\textbf{baseline}}            & \textbf{2.039} & \textbf{3.636} & \textbf{6.628} & \textbf{0.782} & \textbf{0.793} \\
\multicolumn{2}{c}{\textbf{zero wtmk}}           & \textbf{2.027} & \textbf{3.240} & \textbf{3.842} & \textbf{0.666} & \textbf{0.855} \\
\multirow{2}{*}{\textbf{wtmked}} & \textbf{True} & \textbf{2.015} & \textbf{3.493} & \textbf{3.460} & \textbf{0.671} & \textbf{0.843} \\
                                 & False         & 2.017          & 5.392          & 5.311          & 0.665          & 0.836          \\ \hline
\end{tabular}
\caption{Model performance under MNIST with different $f_1(t)$ settings.}
\label{tab:zero_f1}
\end{table}

\begin{figure}[t]
\centering
\includegraphics[width=\linewidth]{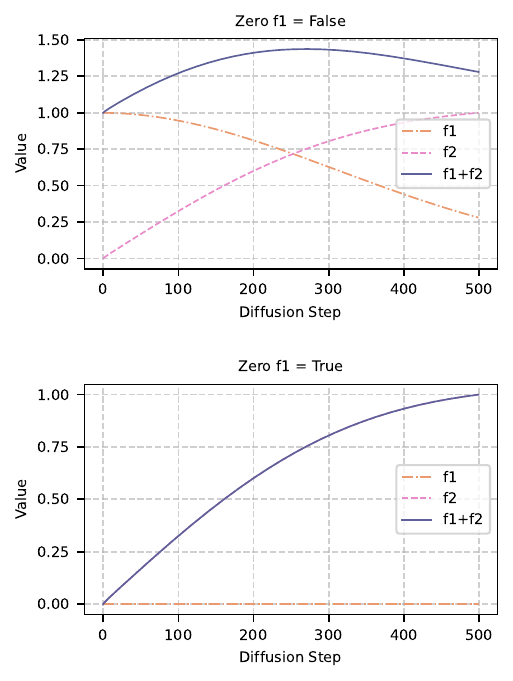}
\caption{The trend of parameters changing with diffusion step.} \label{fig:zero_f1}
\end{figure}

\section{Unsuitable Watermark Timestep}
\label{app:D}

As shown in \cref{fig:fashion_unsuitableTA}, when we set $t_A = (1/2)T$ for FashionMNIST, the models tend to restore the noise to the samples with more texture, even for the zero watermarked models. We also test different $\gamma$ settings under this situation. $\gamma = 1.0$ under zero watermarked model means we totally train the model same as baseline except the $t_A$ setting.
\begin{figure}[!htbp]
\centering
\includegraphics[width=\linewidth]{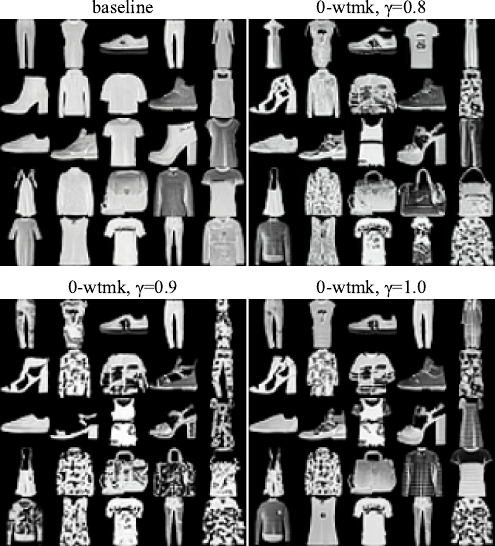}
\caption{Comparison of FashionMNIST samples generated by baseline model and zero watermarked models under different $\gamma$ settings with same $t_A = (1/2) T$.} \label{fig:fashion_unsuitableTA}
\end{figure}

Result reveals the large distribution shifts in the model sampling process is caused by unsuitable watermark timestep. When the dataset contains samples that differ greatly, early $t_A$ results in the model not being able to distinguish certain features of different samples very well.

\end{document}